\documentclass[preprint,12pt]{colt2021} 

\title[Extending Universal Approximation Guarantees]{Extending Universal Approximation Guarantees: A Theoretical Justification for the Continuity of Real-World Learning Tasks }
\usepackage{times}
\usepackage{tikz-cd}
\usepackage{bbm}
\usepackage{comment}

\coltauthor{%
 \Name{Naveen Durvasula} \Email{ndurvasula@berkley.edu}
}

\newcommand{\R}{\mathbb{R}}

\newcommand{\bra}[1]{\left[#1\right]}
\newcommand{\pa}[1]{\left(#1\right)}
\newcommand{\sett}[1]{\left\{ #1 \right\}}
\newcommand{\Q}{\mathbb{Q}}

\newcommand{\N}{\mathbb{N}}

\newcommand{\pr}[1]{\Pr\bra{#1}}

\newcommand{\ex}[1]{\mathrm{E}\bra{#1}}

\newcommand{\B}[1]{\boldsymbol{\mathbf{#1}}}

\newcommand{\norm}[1]{\left\lVert #1 \right\rVert}
\newcommand{\abs}[1]{\left| #1 \right|}

\newcommand{\Bor}{\mathcal{B}}

\newcommand{\F}{\mathcal F}

\makeatletter
\let\Ginclude@graphics\@org@Ginclude@graphics 
\makeatother
\begin{document}

\maketitle

\begin{abstract}
  Universal Approximation Theorems establish the density of various classes of neural network function approximators in $C(K, \R^m)$, where $K \subset \R^n$ is compact. In this paper, we aim to extend these guarantees by establishing conditions on learning tasks that guarantee their continuity. We consider learning tasks given by conditional expectations $x \mapsto \ex{Y \mid X = x}$, where the learning target $Y = f \circ L$ is a potentially pathological transformation of some underlying data-generating process $L$. Under a factorization $L = T \circ W$ for the data-generating process where $T$ is thought of as a deterministic map acting on some random input $W$, we establish conditions (that might be easily verified using knowledge of $T$ alone) that guarantee the continuity of practically \textit{any} derived learning task $x \mapsto \ex{f \circ L \mid X = x}$. We motivate the realism of our conditions using the example of randomized stable matching, thus providing a theoretical justification for the continuity of real-world learning tasks.
\end{abstract}


\section{Introduction}\label{sec:intro}

The expressive capabilities of neural network architectures have historically been understood through Universal Approximation Theorems. The classical result \citep{cybenko1989approximation, hornik1989multilayer, pinkus1999approximation} establishes the density of neural network function approximators with arbitrary width and bounded depth in $C(K, \R)$, where $K \subset \R^n$ is a compact set. Such density has more recently been established for classes of neural network function approximators of arbitrary depth and bounded width \citep{lu2017expressive, hanin2017approximating, kidger2020universal, park2020minimum}. 

The usefulness of Universal Approximation Theorems in understanding the practical success of neural networks hinges on a key assumption: that real-world learning tasks are continuous. In this paper, we provide a theoretical justification for this intuitive assumption. We consider tasks where the learner aims to predict a conditional expectation $x \mapsto \ex{Y \mid X = x}$. Such tasks commonly arise in both the regimes of regression and classification. In the regression case, the conditional expectation is the well-known minimizer of the mean-square error loss -- a loss commonly used in practice. In the classification case, when $Y$ is an indicator variable for one of $k$ disjoint events, the conditional expectation is equal to classification likelihood. We analyze the conditions for such learning tasks to be continuous as a function of $x$. 

In many cases, the learning target $Y$ can be thought of as some (potentially ill-behaved) transformation of a \textit{data-generating process} $L$. For example, in the well-known UCI Adult dataset, the learner aims to predict the odds of a person with features $X$ making an income above $\$50,000$. In this case, there is an underlying random variable $L$ denoting income. We then aim to learn $x \mapsto \ex{f \circ L \mid X = x}$, where $f$ denotes the indicator function for whether $L \ge 50000$. The function $f$, being an indicator function, is not continuous. However, the map $x \mapsto \ex{f \circ L \mid X = x}$ can empirically be seen to be continuous (over continuous features $X$), and indeed may be approximated well by continuous function approximators such as neural networks. 

In this paper, we seek to further justify the empirical success of neural network function approximators by explaining this behavior. We place a realistic regularity constraint on data-generating processes $L$ and show that any derived learning task $x \mapsto \ex{Y \mid X = x}$, where $Y = f \circ L$ for some nearly arbitrary $f$, is continuous. By applying existing universal approximation guarantees, we may establish the approximability of these tasks by neural network function approximators. 

To illustrate the nuance in this problem, we exhibit two seemingly similar data-generating processes $L_1$ and $L_2$ with different continuity properties. Let $X$ be a globally supported real-valued random variable, and let $R \sim U[0, 1]$ be independent from $X$. Let $L_1 := X + R$ denote the sum, and let $L_2 := XR$ denote the product. Finally, let $\operatorname{frac}(x) := x - \lfloor x \rfloor$ map real numbers to their fractional part. We plot the conditional expectations of $\operatorname{frac} \circ L_1$ and $\operatorname{frac} \circ L_2$.

\begin{figure}[h!]
    \centering
    \begin{minipage}{.49\textwidth}
        \centering
        \includegraphics[width=0.9\textwidth]{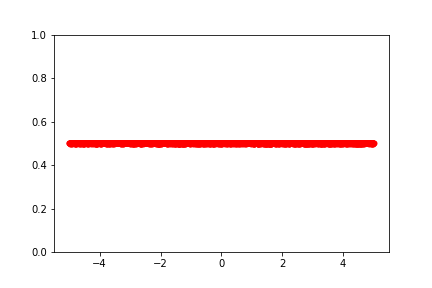}
        \caption{$\ex{\operatorname{frac}(X + R) \mid X = x}$ vs $x$}
        \label{fig:l1}
    \end{minipage} \begin{minipage}{0.49\textwidth}
        \centering
        \includegraphics[width=0.9\textwidth]{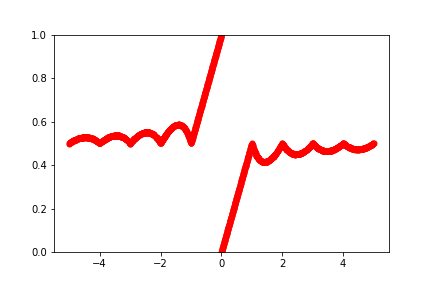}
        \caption{$\ex{\operatorname{frac}(XR) \mid X = x}$ vs $x$}
        \label{fig:l2}
    \end{minipage}
\end{figure}

As one would expect, the map $x \mapsto \ex{\operatorname{frac} \circ L_1 \mid X = x}$ is constant at $0.5$ (as depicted in Figure \ref{fig:l1}). We are simply taking the average fractional part of a $U[0,1]$ variable, and thus recover its expectation. However, the map $x \mapsto \ex{\operatorname{frac} \circ L_2 \mid X = x}$ has a clear discontinuity at $x=0$ (as depicted in Figure \ref{fig:l2}). At face value, the variables $L_1$ and $L_2$ appear similar: both maps $x \mapsto \ex{L_1 \mid X = x} = x + \frac12$ and $x \mapsto \ex{L_2 \mid X = x} = \frac{x}{2}$ are continuous functions of $x$. However, the continuity of the variable $L_1$ is more ``robust'' than that of $L_2$. Indeed, $L_1$ satisfies the \textit{continuous-regularity} property we define in Section \ref{sec:regularity}, and any essentially bounded transformation $f \circ L_1$ will also have continuous conditional expectation with respect to $x$. As demonstrated above, the same cannot be said about $L_2$. We return to this example in Section \ref{sec:apps}.

In Section \ref{sec:prelims}, we give a more formal measure-theoretic description of our problem. In Section \ref{sec:related}, we give an overview of related work. In Section \ref{sec:factor}, we give a useful factorization for data-generating processes $L$, and show that real-world data-generating process may be factored in this way. In Section \ref{sec:regularity}, we give our regularity constraint in terms of the aforementioned factorization, and show that it implies that learning tasks $x \mapsto \ex{f \circ L \mid X = x}$ are continuous, where $f$ is a nearly arbitrary function. Our constraint comes in two flavors, depending on whether $L$ is a discrete or continuous random variable. In Section \ref{sec:apps}, we demonstrate that our condition can be easy to show in practice, and therefore may be used to prove that a \textit{specific} learning task is continuous, even if knowledge of the underlying randomness is limited. We use the example of randomized stable matching, and show that learning tasks derived from match data are continuous.

\section{Preliminaries} \label{sec:prelims}

We restate the classical Universal Approximation Theorem \citep{cybenko1989approximation, hornik1989multilayer, pinkus1999approximation}

\begin{theorem}
Let $\rho: \R \to \R$ be any continuous function, and let $\mathcal N_n^\rho$ denote the class of feedforward neural networks with activation $\rho$, $n$ neurons in the input layer, one neuron in the output layer, and one hidden layer with an arbitrary number of neurons. Let $K \subset \R^n$ be compact. Then, $\mathcal N_n^\rho$ is dense in $C(K, \R)$.
\end{theorem}

In the analyses due to \cite{lu2017expressive}, \cite{hanin2017approximating}, \cite{kidger2020universal}, and \cite{park2020minimum}, the density of deep narrow networks is established in $C(K, \R^m)$ with respect to the uniform norm. We show that a broad family of learning tasks belong to $C(K, \R^m)$, and are thus approximable by neural network function approximators. More generally, we study the conditions necessary for a learning task to belong to $C(K, \R)$, where $K$ is a Radon space. As $C(K, \R^m) = C(K, \R)^m$, our analysis extends to the vector-valued case in a straightforward manner. 

We work in the probability space $(\Omega, \F_\Omega, \mu)$, and assume that learning tasks take the form of a conditional expectation. We assume that $\Omega$ is Radon.    

\begin{definition}[Conditional Expectation]\label{def:condexp}
Let $Y: \Omega \to \R$ and $\mathcal G \subset \F_{\Omega}$ be a sub-$\sigma$-algebra of $\F_{\Omega}$. A conditional expectation $\ex{Y \mid \mathcal G}$ is any $\mathcal G$-measurable real-valued function that satisfies the property
\begin{equation}
    \int_G \ex{Y \mid \mathcal G}d\mu = \int_G Y d\mu
\end{equation}
for all $G \in \mathcal G$. Conditional expectations $\ex{Y \mid X} := \ex{Y \mid \sigma(X)}$ may also be defined relative to a random variable $X$ by applying the preceding definition to the sub-$\sigma$-algebra generated by the random variable. Conditional expectations exist, and are $\mu$-almost everywhere uniquely determined. 
\end{definition}

Formally, we consider learning tasks $\ex{Y \mid X}: K \to \R$ where $Y: \Omega \to \R$ is a random variable as in Definition \ref{def:condexp}, and $X: \Omega \to K$ is a random variable where $K$ is a separable, complete metric space. The random variable $X$ is measurable with respect to the Borel $\sigma$-algebra $\Bor(K)$. 

\begin{definition}[Regular Conditional Probability]\label{def:condprob}
Let $X: \Omega \to K$ be a random variable over the probability space $(\Omega, \F_\Omega, \mu)$. Regular conditional probabilities are a family of probability measures $\sett{\mu_x}_{x \in K}$ over the $\sigma$-algebra $\F_\Omega$ such that for any $S \in \F_\Omega$ and $A \in \Bor(K)$,
\begin{equation}
    \mu(S \cap X^{-1}(A)) = \int_{A}\mu_x(S)d\bra{\mu \circ X^{-1}}(x)
\end{equation}
Further, for any $S \in \F_\Omega$, $x \mapsto \mu_x(S)$ is a $\Bor(K)$-measurable function. The disintegration theorem states that if $\Omega$ and $K$ are Radon spaces, then regular conditional probabilities exist and are $(\mu \circ X^{-1})$-almost everywhere uniquely determined. 
\end{definition}

In our main result, we show the continuity of learning tasks $\ex{Y \mid X}$ in terms of a latent-space model. We sometimes refer to $X$ as the \textit{input}, and $Y$ as the \textit{output}. We assume the existence of a \textit{data-generating process} (formally, a random variable) $L: \Omega \to \Theta$ to a measure space $(\Theta, \F_{\Theta})$, and say that a learning task $\ex{Y \mid X}: K \to \R$ is \textit{derived} from a data-generating process $L$ if we may write $Y = f \circ L$ for some $f: \Theta \to \R$ such that $\sup \sett{\norm{f}_{L^{\infty}(\mu_x \circ L^{-1})} \mid x \in K} < \infty$. In other words, a learning task is derived from a data-generating process $L$ if the target $Y$ we aim to learn is given by a transformation of $L$. This transformation can be pathological (e.g. highly discontinuous), so long as it has an essential supremum with respect to the pushforward of any regular conditional probability given by the input $X$. 

We identify realistic constraints on data-generating processes and show that these constraints imply that \textit{any} learning task derived from the process lies in $C(K,\R)$. As the family of learning tasks derived from a given data-generating process can be immensely large, our result, in conjunction with existing Uniform Approximation Theorems, demonstrates the approximability of a vast collection of learning tasks by neural network function approximators. 

\section{Related Work}\label{sec:related}

The continuity of conditional expectation operator has been well-studied as a function of its two arguments (the random variable and sub-$\sigma$-algebra). The Martingale convergence theorems \citep{billingsley1965ergodic, doob1953stochastic, loeve1963probability} place conditions on sequences of random variables $Y_n$ and/or sub-$\sigma$-algebras $\mathcal F_n$ such that if $Y_n \to Y$ and/or $\mathcal F_n \to \mathcal F$, then the corresponding conditional expectation functions $\ex{Y_n \mid \F_n}$ converge.
\begin{theorem}[Martingales]\label{thm:martingales}
If $\sett{\mathcal F_n}_{n \in \N}$ is a sequence of sub-$\sigma$-algebras that is monotone increasing (i.e. $\mathcal F_n \subseteq \mathcal F_{n+1}$ for any $n$), then 
\[\ex{Y \mid \mathcal F_n} \xrightarrow{L^p(\mu)} \ex{Y \mid \bigvee_{n=1}^\infty \mathcal F_n}\]
for every $Y \in L^p(\mu)$ and $1 \le p \le \infty$, where $\bigvee_{n=1}^\infty \mathcal F_n$ is the $\sigma$-algebra generated by $\bigcup_{n=1}^\infty \mathcal F_n$.
\end{theorem}
Later work by \cite{boylan1971equiconvergence}, \cite{fetter1977continuity}, and \cite{alonso1998lp} prove related results. However, these results do not resolve the continuity of the conditional expectation as a real-valued function as we aim to in this paper (i.e. if $\ex{Y \mid X} \in C(K, \R)$). Rather, as stated in Theorem \ref{thm:martingales}, these results establish the convergence of a sequence of conditional expectation functions in $L^p(\mu)$.  

In \cite{dolera2020lipschitz, dolera2020uniform}, conditions are given for the uniform continuity of regular conditional probabilities $\mu_x$ relative to a modulus of continuity. In theory, these conditions can be used to establish the continuity of $\ex{Y \mid X}$ given some regularity constraints on $Y$. Although the conditions given in \cite{dolera2020lipschitz, dolera2020uniform} are amenable for analysis in the context of the well-posedness of Bayesian inference \citep{stuart2010inverse, dashti2013bayesian, cotter2009bayesian, iglesias2014well} and Bayesian consistency \citep{diaconis1986consistency, ghosal2017fundamentals}, they are difficult to interpret in our setting. Our condition takes a very different form from that given in \cite{dolera2020lipschitz, dolera2020uniform}: whereas we propose a factorization constraint on a data-generating process $L$, they propose an integrability constraint directly on the regular conditional probabilities. 

\section{A Factorization for Data-Generating Processes}\label{sec:factor}

In our model, learning targets are given by (nearly arbitrary) transformations $f: \Theta \to \R$ applied to data-generating processes $L: \Omega \to \Theta$. In this section, we introduce a useful factorization for the data-generating process $L$, where we think of the process as an operation on the input $X$ in addition to some extra noise $R$. The noise $R: \Omega \to \Gamma$ is a random variable to some measure space $(\Gamma, \F_\Gamma)$. Formally, we assert that there exists a measure space $(\Gamma, \F_\Gamma)$ and measurable maps $W: \Omega \to K \times \Gamma$ and $T: K \times \Gamma \to \Theta$ such that the diagram in Figure \ref{fig:diag} commutes.

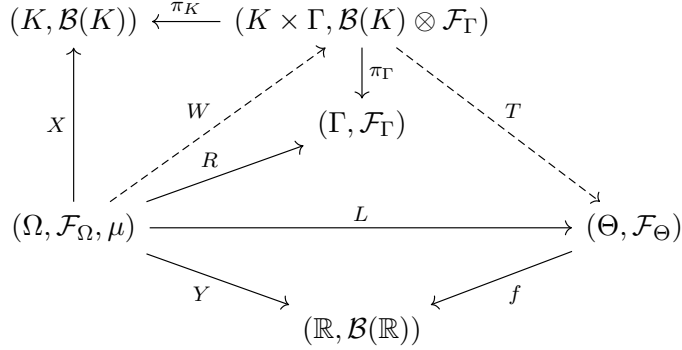
\begin{figure}[h!]
\centering
\begin{tikzcd}
{(K, \mathcal B(K))}                                                                               & {(K \times \Gamma, \mathcal{B}(K) \otimes \mathcal F_{\Gamma})} \arrow[rdd, "T", dashed] \arrow[l, "\pi_K"'] \arrow[d, "\pi_\Gamma"] &                                                                   \\
                                                                                                                     & {(\Gamma, \mathcal F_\Gamma)}                                                                                                                                &                                                                   \\
{(\Omega, \F_\Omega, \mu)} \arrow[uu, "X"] \arrow[rr, "L"] \arrow[rd, "Y"'] \arrow[ruu, "W", dashed] \arrow[ru, "R"] &                                                                                                                                                                                & {(\Theta, \mathcal F_{\Theta})} \arrow[ld, "f"] \\
                                                                                                                     & {(\mathbb R, \mathcal B(\mathbb R))}                                                                                                                         &                                                                  
\end{tikzcd}

\caption{\label{fig:diag}The $T \circ W$ factorization for data-generating processes}
\end{figure}

The map $W$ transforms the input to the data-generating process (i.e. the original outcome $\omega$) to $(X(\omega), R(\omega))$, the values of the the random variables $X$ and $R$. The map $T$ then takes this data to the latent space $\Theta$. It is always possible to trivially factorize $L$ in this way by letting the ``additional'' noise be all of the randomness we initially started with (i.e. by letting $(\Gamma, \F_\Gamma) = (\Omega, \F_\Omega)$, $R: \omega \mapsto \omega$, and $T: (x, \omega) \mapsto L(\omega)$). 

We say that a factorization $L = T \circ W$ is \textit{decomposable} if the additional noise $R = \pi_\Gamma \circ W$ is independent from $X$. Formally, in a decomposable factorization, we assert that the pushforward measure $\mu \circ W^{-1}$ over $(K \times \Gamma, \mathcal{B}(K) \otimes \mathcal F_{\Gamma})$ can be decomposed as a product measure
\begin{equation}
    \mu \circ W^{-1} = \mu \circ X^{-1} \times \mu \circ R^{-1}
\end{equation}
Although this might appear to be a somewhat restrictive condition at first glance, we show that if we start with some arbitrary initial factorization $L = T \circ W$ where the additional noise $R$ conditioned on $X$ is a continuous vector-valued random variable (as in many practical settings), then we may construct a decomposable factorization $L = T' \circ W'$.  
\begin{theorem}\label{thm:correction}
Let $L: \Omega \to \Theta$ be a data-generating process with a factorization $L = T \circ W$ as in Figure \ref{fig:diag}, and suppose that $\Gamma \subset \R^k$. Further, suppose that for every $i \in [k]$ and $x \in K$, the cumulative distribution function of $R_i$ (the $i$th component of $R = \pi_\Gamma \circ W$) conditioned on $X$ and $R_1, \dots, R_{i-1}$ is continuous. Then, there exists a decomposable factorization $L = T'\circ W'$. 
\end{theorem}

\textbf{Sketch of Proof (See Appendix \ref{proof:correction} for full proof). }The core idea behind this argument is inverse CDF sampling: a uniform random variable can be transformed to any continuous real-valued random variable by using the inverse of the cumulative distribution function. Letting $R_1, \dots, R_k$ denote the components of $R$, we apply an inductive process to invertibly transform $R$ to a collection of $k$ i.i.d uniform random variables $R'$ by letting the $i$th component $R'_i$ be given by the conditional CDF of $R_i$ conditioned on the values of $R_1,\dots,R_{i-1}$ and $X$. By assumption, these CDFs are continuous and therefore invertible over the support of $R$. Thus, there exists some invertible collection of maps $\sett{I_x}_{x \in K}$ such that $R' = I_X(R)$. We then let $W': \omega \mapsto (X(\omega), I_{X(\omega)}(R(\omega)))$, and $T': (x , r') \mapsto T(x, I_x^{-1}(r'))$. Clearly, $L = T' \circ W'$, and as $R'$ is independent of $X$, it is also decomposable. $\square$

In a decomposable factorization, the map $W$ can be seen as a sort of \textit{whitening} operation: the initial outcome $\omega$ is split into the input $X$ and a component $R$ independent to the input. The map $T$ then transforms these independent components to the space $\Theta$. It follows that once we condition on the value of the input $X$, all randomness in the $L$ (and thus the output $Y$) is then encapsulated by the random variable $R$. The following Lemma makes this relationship concrete. For each $x \in K$, we let $T_x: \Gamma \to \Theta$ denote that map that takes each $r\in \Gamma$ to $T(x, r)$.

\begin{lemma}\label{lem:decompose}
Let $L = T \circ W$ be decomposable.  For any $x \in K$,
\[\mu_x \circ L^{-1} = \mu \circ R^{-1} \circ  T_x^{-1}\]
almost everywhere, where $\mu_x$ is a regular conditional probability for $\mu$ over $X$. 
\end{lemma}

See Appendix \ref{proof:decompose} for the proof of this Lemma. This result is the reason why decomposability is a desirable property: it allows us to understand the conditional probabilities $\mu_x \circ L^{-1}$ through the lens of the map $T$. The maps $W$ and $T$ give two disjoint pieces of information about $L$: $W$ can be thought of as providing the underlying randomness in $L$, whereas $T$ can be thought of as a deterministic map that takes the random data and maps it $\Theta$. Using this perpsective, Lemma \ref{lem:decompose} has a nice interpretation: in a decomposable factorization $T \circ W$, conditioning $L$ \textit{probabilistically} on the variable $X$ is the same as conditioning $T$ \textit{deterministically} on $X$. Thus, deterministic properties about $T$ that might be known before-hand can be used to analyze the probabilistic properties of $L$. In the next section, we make this intuition concrete: we give constraints on $T$ that may be verified under minimal assumptions on $W$, and show that these constraints imply the continuity of derived learning tasks. 

\section{Discrete and Continuous Regularity}\label{sec:regularity}

Decomposable factorizations $L = T \circ W$ satisfy a constraint on $W$. We now give a second constraint on $T$ and show that this constraint (in conjuction with decomposability) implies the continuity of any derived learning task $\ex{f \circ L \mid X}$. We provide the full proofs of these claims in this section as they constitute the central contribution of this paper. Our constraint comes in two flavors corresponding to whether $L$ is a discrete or continuous random variable. We first focus on the discrete case, where we think of the latent space $\Theta$ as discrete or otherwise non-metrizable. 

\begin{definition}[Discrete-Regular Factorization]\label{def:disc}
A factorization $L = T \circ W$ is \textit{discrete-regular} if for any $x_0 \in K$, 
\[\lim_{x \to x_0} \bra{\mu \circ R^{-1}} \pa{\sett{r \in \Gamma \mid T_x(r) \ne T_{x_0}(r)}} = 0\]
\end{definition}

Intuitively, this condition asserts that when we fix the ``extra'' randomness $r$, the probability that the output of the data-generating process $L$ changes when the input $x$ is perturbed slightly goes to zero. As we show in Section \ref{sec:apps}, this condition can be verified using only knowledge of the map $T$, while placing minimal constraints on the random variable $R$. We now show that for any discrete-regular, decomposable factorization $L = T \circ W$ and any $x, x_0 \in K$, the conditional probability measure $\mu_x \circ L \to \mu_{x_0} \circ L$ converges in the strong topology as $x \to x_0$.

\begin{proposition}\label{prop:disc}
Let $L = T\circ W$ be a discrete-regular, decomposable factorization. Then, for any $x_0 \in K$, the total variation goes to zero in the limit
\[\lim_{x \to x_0} \sup \sett{\abs{\bra{\mu_x \circ L^{-1}}(S) - \bra{\mu_{x_0} \circ L^{-1}}(S)} \mid S \in \F_\Theta } = 0\]
\end{proposition}
\begin{proof}
We show that the absolute differences $\abs{\bra{\mu_x \circ L^{-1}}(S) - \bra{\mu_{x_0} \circ L^{-1}}(S)}$ converge uniformly over subsets $S \in \F_\Theta$. Formally, we aim to show that for all $x_0 \in K$ and $\epsilon > 0$, there exists a neighborhood $N \subset K$ such that for all $x \in N$, and subsets $S \in \F_\Theta$, \[\abs{\bra{\mu_x \circ L^{-1}}(S) - \bra{\mu_{x_0} \circ L^{-1}}(S)} < \epsilon\]

As $L = T \circ W$ is a discrete-regular factorization, we have that there exists a neighborhood $N \subset K$ such that
\[\bra{\mu \circ R^{-1}} \pa{\sett{r \in \Gamma \mid T_x(r) \ne T_{x_0}(r)}} < \epsilon\]
for all $x \in N$. Since the factorization $L = T \circ W$ is decomposable, we have by Lemma \ref{lem:decompose} that for any such measurable subset $S$,
\begin{align*}
\abs{\bra{\mu_x \circ L^{-1}}(S) - \bra{\mu_{x_0} \circ L^{-1}}(S)} &= \abs{\bra{\mu \circ R^{-1} \circ T_x^{-1}}(S) - \bra{\mu \circ R^{-1} \circ T_{x_0}^{-1}}(S)}\\
&\le \bra{\mu \circ R^{-1}}\pa{T_x^{-1}(S) \Delta T_{x_0}^{-1}(S)}
\end{align*}
where $\Delta$ denotes the symmetric difference. Observe that if some $r \in \Gamma$ lies in the set $T_x^{-1}(S) \Delta T_{x_0}^{-1}(S)$, then we must have that $T_x(r) \ne T_{x_0}(r)$. We can therefore say that 
\begin{align*}
    \bra{\mu \circ R^{-1}}\pa{T_x^{-1}(S) \Delta T_{x_0}^{-1}(S)} \le  \bra{\mu \circ R^{-1}} \pa{\sett{r \in \Gamma \mid T_x(r) \ne T_{x_0}(r)}} < \epsilon
\end{align*}
by construction, thus proving the result.
\end{proof}

Since no metric is placed on $\Theta$, the discrete-regularity condition places a ``hard'' constraint on the data-generating process $L$: as $x$ approaches $x_0$, the corresponding outcomes $T_x$ and $T_{x_0}$ must be \textit{equal} with increasingly large probability when the extra noise is fixed. In the continuous case, where we assert that $\Theta \subset \R^d$, we can instead place a ``soft'' constraint on $L$ by requiring that as $x$ approaches $x_0$, the corresponding outcomes $T_{x_0}$ and $T_x$ become arbitrarily \textit{close} with increasingly large probability. 

\begin{definition}[Continuous-Regular Factorization]
A factorization $L = T \circ W$ is \textit{continuous-regular} if for all $x \in K$, $\mu_x \circ L^{-1} \ll \lambda^d$ is absolutely continuous with respect to the $d$-dimensional Lebesgue measure $\lambda^d$ with bounded Radon-Nikodym derivative $\frac{d\bra{\mu_x \circ L^{-1}}}{d\lambda^d} \le D$, and for any $\tau > 0$
\[\lim_{x \to x_0} \bra{\mu \circ R^{-1}} \pa{\sett{r \in \Gamma \mid \norm{T_x(r) - T_{x_0}(r)} \ge \tau}} = 0\]
\end{definition} 

We now similarly show that for any continuous-regular, decomposable factorization $L = T \circ W$ and any $x, x_0 \in K$, the conditional probability measure $\mu_x \circ L \to \mu_{x_0} \circ L$ converges in the strong topology as $x \to x_0$.

\begin{proposition}\label{prop:cont}
Let $L = T\circ W$ be a continuous-regular, decomposable factorization. Then, for any $x_0 \in K$, the total variation goes to zero in the limit:
\[\lim_{x \to x_0} \sup \sett{\abs{\bra{\mu_x \circ L^{-1}}(S) - \bra{\mu_{x_0} \circ L^{-1}}(S)} \mid S \in \F_\Theta } = 0\]
\end{proposition}

\begin{proof}
The intuition behind this argument is quite similar to that of Proposition \ref{prop:disc}, but much more care is necessary to carry out the argument. As before, we aim to show that for any $x_0 \in K$ and $\epsilon > 0$, there exists a neighborhood $N \subset K$ about $x_0$ such that for all $x \in N$ and $S \in \F_\Theta$,
 \[\abs{\bra{\mu_x \circ L^{-1}}(S) - \bra{\mu_{x_0} \circ L^{-1}}(S)} < \epsilon\]
However, rather than showing this directly for \textit{all} subsets $S \in \F_\Theta \subseteq \Bor(\R^d)$, we first establish uniform convergence for subsets $J$ that are \textit{Jordan-measurable}.

\begin{definition}[Jordan-measurability]
A measurable subset $J \in \Bor(R^d)$ is Jordan-measurable if
\[\lambda^d(\partial J) = \lambda^d\pa{\sett{\theta \in \R^d \mid d(\theta, J) = 0}} = 0\]
where $d(\theta,J) := \inf \sett{\norm{\theta - s} \mid s \in J} = 0$ denotes the distance from the point $\theta$ to the set $J$. Although boxes, balls and other ``simple'' sets are Jordan-measurable, other Borel sets such as $\Q$ or the ``fat'' Cantor set are not Jordan-measurable.
\end{definition}

We prove a key intermediate fact about Jordan-measurable subsets. We show that no ``well-behaved'' probability measure can assign large measure to a thin annulus about a Jordan-measurable set. 
\begin{lemma}\label{lem:jordan}
Let $J \in \F_\Theta$ be a Jordan-measurable subset, and let $\sett{\nu_t}_{t \in T}$ be a family of probability measures, each satisfying $\nu_t \ll \lambda^d$ and $\frac{d\nu_t}{d\lambda^d} \le D$. Then,
\[\lim_{\delta \to 0} \sup \sett{\nu_t\sett{\theta \in \R^d \setminus J \mid d(\theta, J) < \delta} \mid t \in T} = 0\]
\end{lemma}

See Appendix \ref{proof:jordan} for the proof. We now show that for any $x_0 \in K$ and $\epsilon > 0$, there exists a neighborhood $N \subset K$ about $x_0$ such that for all $x \in N$ and Jordan-measurable $J \in \F_\Theta$, $\abs{\bra{\mu_x \circ L^{-1}}(J) - \bra{\mu_{x_0} \circ L^{-1}}(J)} < \epsilon$. By continuous-regularity, the measures $\sett{\mu_x \circ L^{-1}}_{x \in K}$ are absolutely continuous with respect to the $d$-dimensional Lebesgue measure $\lambda^d$ and have bounded Radon-Nikodym derivative $\frac{d\bra{\mu_x \circ L^{-1}}}{d\lambda^d} \le D$. Thus, by Lemma \ref{lem:jordan}, there exists a $\tau > 0$ such that for all $x \in K$,
\[\bra{\mu_x \circ L^{-1}}\pa{\sett{\theta \in \R^d \setminus J \mid d(\theta, J) < \tau}} < \frac{\epsilon}{2}\]
By continuous-regularity, we also have that there exists a neighborhood $N \subset K$ such that for all $x \in N$,
\[\bra{\mu \circ R^{-1}} \pa{\sett{r \in \Gamma \mid \norm{T_x(r) - T_{x_0}(r)}\ge \tau}} < \frac{\epsilon}{2}\]
We then have, applying Lemma \ref{lem:decompose}, that
\begin{align*}
    \abs{\bra{\mu_x \circ L^{-1}}(J) - \bra{\mu_{x_0} \circ L^{-1}}(J)} &= \abs{\bra{\mu \circ R^{-1} \circ T_x^{-1}}(J) - \bra{\mu \circ R^{-1} \circ T_{x_0}^{-1}}(J)}\\
&\le \bra{\mu \circ R^{-1}}\pa{T_x^{-1}(J) \Delta T_{x_0}^{-1}(J)}\\\
&= \bra{\mu \circ R^{-1}}\pa{T_x^{-1}(J) \setminus T_{x_0}^{-1}(J)} + \bra{\mu \circ R^{-1}}\pa{T_{x_0}^{-1}(J) \setminus T_{x}^{-1}(J)}
\end{align*}
Consider the following two subsets of $\Theta$
\[J_{x_0} := T_{x_0}\pa{T_x^{-1}(J) \setminus T_{x_0}^{-1}(J)} \qquad J_{x} := T_{x}\pa{T_{x_0}^{-1}(J) \setminus T_{x}^{-1}(J)}\]
Observe that both of these subsets are disjoint from $J$. Further, $r \in T_x^{-1}(J) \setminus T_{x_0}^{-1}(J)$, then $T_{x_0}(r) \in J_{x_0}$ and $T_x(r) \in J$. Jimilarly, if $r \in T_{x_0}^{-1}(J) \setminus T_{x}^{-1}(J)$, then $T_x(r) \in J_x$ and $T_{x_0}(r) \in J$. We can then apply Lemma \ref{lem:decompose} again, to see that
\begin{align*}
    \bra{\mu \circ R^{-1}}\pa{T_x^{-1}(J) \setminus T_{x_0}^{-1}(J)} + \bra{\mu \circ R^{-1}}\pa{T_{x_0}^{-1}(J) \setminus T_{x}^{-1}(J)} =\bra{\mu_{x_0} \circ L^{-1}}\pa{J_{x_0}} + \bra{\mu_x \circ L^{-1}}\pa{J_x}
\end{align*}
which we can then split up as the sum of the terms
\begin{align*}
&\bra{\mu_{x_0} \circ L^{-1}}\pa{\sett{\theta \in  J_{x_0} \mid d(\theta, J) < \tau}} + \bra{\mu_{x} \circ L^{-1}}\pa{\sett{\theta \in  J_x \mid d(\theta, J) < \tau}}\\
&\le \sup \sett{\bra{\mu_x \circ L^{-1}}\pa{\sett{\theta \in \R^d \setminus J \mid d(\theta, J) < \tau}} \mid x \in K}\\
&< \frac{\epsilon}{2}
\end{align*}
and the terms
\begin{align*} 
&\bra{\mu_{x_0} \circ L^{-1}}\pa{\sett{\theta \in  J_{x_0} \mid d(\theta, J) \ge \tau}} + \bra{\mu_{x} \circ L^{-1}}\pa{\sett{\theta \in  J_{x_0} \mid d(\theta, J) \ge \tau}}\\
&= \bra{\mu \circ R^{-1}}\pa{\sett{r \in T_x^{-1}(J) \setminus T_{x_0}^{-1}(J) \mid d(T_{x_0}(r),J) \ge \tau}} \\
&\qquad \qquad + \bra{\mu \circ R^{-1}}\pa{\sett{r \in T_{x_0}^{-1}(J) \setminus T_{x}^{-1}(J) \mid d(T_{x}(r),J) \ge \tau}}\\
&\le \bra{\mu \circ R^{-1}} \pa{\sett{r \in \Gamma \mid \norm{T_x(r) - T_{x_0}(r)}\ge \tau}}\\
&< \frac{\epsilon}{2}
\end{align*}
It thus follows that for all $x \in N$, $\abs{\bra{\mu_x \circ L^{-1}}(J) - \bra{\mu_{x_0} \circ L^{-1}}(J)} < \epsilon$, whence we have that for any $x_0$,
\begin{equation} \label{eq:partial}
\lim_{x \to x_0} \sup \sett{\abs{\bra{\mu_x \circ L^{-1}}(J) - \bra{\mu_{x_0} \circ L^{-1}}(J)} \mid J \in \F_\Theta, \lambda^d(\partial J) = 0} = 0
\end{equation}
To extend this result to any Borel set $S$, we make use of the following Lemma.
\begin{lemma}\label{lem:extend}
For any Borel set $S \in \Bor(\R^d)$,
\[\inf \sett{\lambda^d\pa{S \Delta J} \mid J \in \Bor(\R^d), \lambda^d(\partial J) = 0} = 0\]
\end{lemma}

See Appendix \ref{proof:extend} for the proof. We now show that for any $\epsilon > 0$ and $x_0 \in K$, there exists a neighborhood $N \subset K$ such that for all $x \in N$ and Borel subsets $S \in \F_\Theta$, $\abs{\bra{\mu_x \circ L^{-1}}(S) - \bra{\mu_{x_0} \circ L^{-1}}(S)} < \epsilon$. Using Equation \ref{eq:partial}, we select a neighborhood $N$ such that for all $x \in N$,
\[\sup \sett{\abs{\bra{\mu_x \circ L^{-1}}(J) - \bra{\mu_{x_0} \circ L^{-1}}(J)} \mid J \in \F_\Theta, \lambda^d(\partial J) = 0} < \frac{\epsilon}{2}\]
Next, we apply Lemma \ref{lem:extend} to find a Jordan-measurable subset $J$ such that $\lambda^d(S \Delta J) < \frac{\epsilon}{2D}$. We then have, by continuous-regularity, that
\begin{align*}
    \abs{\bra{\mu_x \circ L^{-1}}(S) - \bra{\mu_{x_0} \circ L^{-1}}(S)} &\le \abs{\bra{\mu_x \circ L^{-1}}(S\Delta J) - \bra{\mu_{x_0} \circ L^{-1}}(S \Delta J)}\\
    & \qquad \qquad + \abs{\bra{\mu_x \circ L^{-1}}(J) - \bra{\mu_{x_0} \circ L^{-1}}(J)} \\
    & < D \lambda^d(S \Delta J) + \frac{\epsilon}{2}\\
    &< \epsilon
\end{align*}
thus proving the result.
\end{proof}
We now use Propositions \ref{prop:disc} and \ref{prop:cont} to show our main result: that derived learning tasks from data-generating processes with decomposable discrete-regular or continuous-regular factorizations are continuous.
\begin{theorem}\label{thm:main}
Let $L = T \circ W$ be a decomposable discrete-regular or continuous-regular factorization. Then, for any $f: \Theta \to \R$ such that $B := \sup \sett{\norm{f}_{L^{\infty}(\mu_x \circ L^{-1})} \mid x \in K} < \infty$, the conditional expectation $\ex{f \circ L \mid X} \in C(K ,\R)$.
\end{theorem}
\begin{proof}
We show that for any $\epsilon > 0$ and $x_0 \in K$, there exists a neighborhood $N \subset K$ such that for all $x \in N$. $\abs{\ex{f \circ L \mid X}(x_0) - \ex{f \circ L \mid X}(x)} < \epsilon$. By Propositions \ref{prop:disc} and \ref{prop:cont}, we may select a neighborhood $N$ such that for all $x \in N$,
\[\sup \sett{\abs{\bra{\mu_x \circ L^{-1}}(S) - \bra{\mu_{x_0} \circ L^{-1}}(S)} \mid S \in \F_\Theta} < \frac{\epsilon}{B}\]
Expanding using the regular conditional probabilities, we have that
\begin{align*}
    \abs{\ex{f \circ L \mid X}(x_0) - \ex{f \circ L \mid X}(x)} &= \abs{\int_{\Omega} \pa{f \circ L} d\mu_{x_0} - \int_{\Omega} \pa{f \circ L} d\mu_x}\\
    &= \abs{\int_{\Theta} f d\bra{\mu_{x_0} \circ L^{-1}} - \int_{\Theta} fd\bra{\mu_{x} \circ L^{-1}}}\\
    &\le \int_{\Theta} \abs{f} d\abs{\bra{\mu_{x_0} \circ L^{-1}} - \bra{\mu_{x} \circ L^{-1}}}\\
    &< B \cdot \frac{\epsilon}{B}
\end{align*}
whence the desired result follows.
\end{proof}

\section{Applications}\label{sec:apps}
In this section, we show how our constraints may be applied to demonstrate the continuity of real-world learning tasks. We first return to the example illustrated in Figures \ref{fig:l1} and \ref{fig:l2} in Section \ref{sec:intro}. Recall that in this example, $X$ was a globally supported real-valued variable and $R \sim U[0,1]$ was independent from $X$. Notice that the data-generating processes $L_1$ and $L_2$ both have decomposable factorizations. Letting
\[W: \omega \mapsto (X, R) \qquad T_1: (x, r) \mapsto x + r \qquad T_2 (x, r) \mapsto xr\]
we may write $L_1 = T_1 \circ W$ and $L_2 = T_2 \circ W$. As the maps $T_1(\cdot, R)$ and $T_2(\cdot, R)$ are both continuous with probability $1$, both factorizations satisfy the second part of the continuous-regularity constraint. The probablity densities, however, are given by
\[\frac{d(\mu_x \circ L_1^{-1})}{d\lambda}(z) = \begin{cases}1 & z \in [x, x + 1]\\0 & \text{else}\end{cases} \qquad \frac{d(\mu_x \circ L_2^{-1})}{d\lambda}(z) = \begin{cases}\abs{\frac1x} & z \in [0, x]\\0 & \text{else}\end{cases} \]
Thus, only $L_1$ has bounded density when conditioned on $X$: near the point $x = 0$, the conditional density of $L_2$ given $X = x$ can become arbitrarily large about zero and is thus not bounded. Indeed, the discontinuity in the conditional expectation that appears when the function $\operatorname{frac}$ is applied to $L_2$ is at the point $x = 0$. 

As $L_1$ has conditional density bounded by $1$, the factorization $L_1 = T_1 \circ W$ satisfies continuous-regularity. Thus, Theorem \ref{thm:main} guarantees that any essentially bounded $f$ may be applied, and the corresponding conditional expectation $\ex{f \circ L \mid X = x}$ will be continuous. As $T_1(\cdot, r): x \mapsto x + r$ is a continuous function for \textit{any} real $r$, the random variable $R$ may in fact be any independent continuous random variable, and the guarantee from Theorem \ref{thm:main} will still hold! Thus, continuity of any conditional expectation $\ex{f \circ L_1 \mid X}$ can be established given only $T_1$ and some minimal assumptions on $R$. We demonstrate the power of this reasoning in showing that real-world learning tasks are continuous, using stable matching as an example.

\textbf{Continuity of Stable Matching. }In the stable matching problem (first introduced in \cite{gale1962college}), there are sets $S_M$ and $S_W$ consisting of $n$ men and $n$ women, each with preferences over agents of the opposite gender. We aim to find a bijection between the men and women such that no man-woman pair mutually prefers to be matched over their assigned partners. We model real-world matching markets, such as the National Residency Matching Program \citep{roth1984evolution} by assigning each $i \in S_M \cup S_W$ to a \textit{feature vector} $X_i \in \R^n$ and continuous \textit{preference function} $P_i: \R^n \to \R$. We let $i$ prefer $j$ to $k$ if $P_i(X_j) > P_i(X_k)$. We assert that preferences are \textit{strict} -- for any $c \in \R$, $\lim_{\delta \to 0} \lambda^n\pa{x \in \R^n \mid P_i(x) \in [c, c + \delta]} = 0$. That is, no ``ties'' are allowed on sets with positive Lebesgue measure. Finally, we let $L_a$ (for $a \in S_M$) denote the feature vector of man $a$'s match under the standard deferred acceptance algorithm. 

We let the preference functions $P_i$ and feature vectors $X_i$ be randomly generated independently from $X_a$, and assert that each $X_i \ll \lambda^n$ is a continuous random variable. The variable $L_a$ may now be factored decomposably as $T \circ W$ where $W$ denotes the collection of random feature vectors and preference functions, and $T$ denotes the output of the deferred acceptance algorithm. We show that \textit{any} essentially bounded function of the match data $L_a$ has continuous conditional expectation with respect to the value of $X_a$, and is therefore approximable by neural network function approximators. 

\begin{theorem}\label{thm:stable}
The factorization $L_a = T \circ W$ is discrete-regular.
\end{theorem}

\begin{proof}
Letting $R$ denote the collection of preference functions and feature vectors (sans $X_a$), we aim to show (as per Definition \ref{def:disc}) that for any $x_a \in \R^n$
\[\lim_{\norm{\B \delta} \to 0}\pr{T(x_a, R) \ne T(x_a + \B \delta, R)} = 0\]
Observe that the stable match must remain the same if each of the women's preference rankings do not change as a result of the perturbation $\B \delta$. Thus, it suffices to show that
\begin{align*}
    &\lim_{\norm{\B \delta} \to 0}\pr{\bigcup_{(m,w) \in S_M \setminus \sett{a} \times S_W} P_w(X_m) \in [P_w(x_a), P_w(x_a + \B \delta)]} \\
    &= \lim_{\delta \to 0}\pr{\bigcup_{(m,w) \in S_M \setminus \sett{a} \times S_W} P_w(X_m) \in [P_w(x_a), P_w(x_a) + \delta]}\\
    &\le \sum_{(m,w) \in S_M \setminus \sett{a} \times S_W} \lim_{\delta \to 0} \pr{P_w(X_m) \in [P_w(x_a), P_w(x_a) + \delta]}
\end{align*}
where in the second line, we invoke the continuity of the preference functions $P_w$. As the preference functions are strict, $\lim_{\delta \to 0} \lambda^n\pa{x \in \R^d \mid P_w(x) \in [P_w(x_a), P_w(x_a) + \delta]} = 0$, whence it follows, by the absolute continuity of each of the $X_i$ that the above sum goes to zero in the limit as desired.
\end{proof}

A subsequent application of Theorem \ref{thm:main} guarantees that for any essentially bounded $f$, the map $x \mapsto \ex{f \circ L_a \mid X_a=x}$  is continuous. Finally, by applying existing universal approximation guarantees, we have theoretical evidence that any derived learning task from stable match data can be well approximated by a neural network. We emphasize that the distribution over the infinite-dimensional space of preference functions and feature vectors was left nearly arbitrary. Theorem \ref{thm:main} allows easily-verifiable deterministic guarantees on $T$ to translate into a strong probabilistic guarantee on $L$. 

\section{Conclusion}
In this paper, we developed a factorization constraint on data-generating processes $L$, such that for a broad family of real-valued functions $f$, conditional expectations $x \mapsto \ex{f \circ L \mid X=x}$ are continuous. The factorization we describe in Section \ref{sec:factor} allows us to view the random variable $L$ as a deterministic function $T$ that acts on random quantities $W$. In Section \ref{sec:regularity}, we showed that guarantees that might be verified largely using knowledge of $T$ alone may be extend to probabilistic guarantees on the continuity of derived learning tasks. As demonstrated in Section \ref{sec:apps}, our regularity condition can be easy to show, even for systems that have many moving parts. 

We believe that our work provides an extension to existing universal approximation guarantees, and provides some additional insight into the empirical success of neural network function approximators. Indeed, by considering the contrapositive of our main result, we have shown that any learning target $Y$ that is not well-approximated by a neural network cannot be written as $f \circ L$ for \textit{any} well-behaved $L$. Thus, such functions must be more deeply pathological.

We see three main avenues for future work. First, while our constraint makes a guarantee on the continuity of maps $x \mapsto \ex{f \circ L \mid X = x}$, it makes no further guarantees (e.g. Lipschitz continuity, differentiability) that are also relevant to the performance of modern learning algorithms. A tighter constraint on $T$ that provides this guarantee in a similar randomness-agnostic fashion would extend our analysis in a meaningful way. 

We also believe that similar statements can be used to justify assumptions made in other domains, such as manifold learning, where it is assumed that the support of a data-generating process is concentrated about a well-behaved lower-dimensional manifold. Manifold learning algorithms such as UMAP, TSNE, and Spectral Methods \citep{mcinnes2018umap,maaten2008visualizing,belkin2002laplacian} are each built around subtly different assumptions on the manifold. Using the framework we develop in this paper, it might be possible to identify the assumptions that hold more generally and thus improve the performance of these algorithms. 

Finally, we believe that results similar to Theorem \ref{thm:stable} might be easy to show for a variety of other economic processes, such as kidney-exchanges \citep{roth2004kidney}, ride-sharing, and other algorithm-based marketplaces. Theorem \ref{thm:main} in conjunction with existing Universal Approximation Theorems then provides a useful formal guarantee on performance for learning tasks derived from such processes.

\acks{I would like to thank Franklyn Wang and Jacob Stavrianos for their assistance with Theorem \ref{thm:correction} and Lemma \ref{lem:extend} respectively.}

\newpage
\bibliography{refs}

\appendix
\newpage 
\section{Proof of Theorem \ref{thm:correction}}\label{proof:correction}
\textbf{Theorem \ref{thm:correction} }\textit{Let $L: \Omega \to \Theta$ be a data-generating process with a factorization $L = T \circ W$ as in Figure \ref{fig:diag}, and suppose that $\Gamma \subset \R^k$. Further, suppose that for every $i \in [k]$ and $x \in K$, the cumulative distribution function of $R_i$ (the $i$th component of $R = \pi_\Gamma \circ W$) conditioned on $X$ and $R_1, \dots, R_{i-1}$ is continuous. Then, there exists a decomposable factorization $L = T'\circ W'$. }\\

\begin{proof}
Let $\sett{\mu_{x,r_1,\dots,r_i}}_{x \in K, r_1,\dots,r_i \in \R}$ denote the regular conditional probability for $\mu$ given the random variables $(X, R_1, \dots, R_i)$. By premise, we have that the conditional cumulative distribution function for $R_i$
\[F_{x,r_1,\dots,r_{i-1}}(r) := \bra{\mu_{x,r_1,\dots,r_{i-1}} \circ R_i^{-1}}\pa{\sett{z \in \R \mid z \le r}}\]
is continuous. This function is invertible over the support of the distribution 
\[F_{x,r_1,\dots,r_{i-1}}^{-1}(c) := \sup\pa{\sett{z \in \R \mid F_{x,r_1,\dots,r_{i-1}}(z) \le c}}\]
For each $x \in K$, we define the map $I_x: \R^k \to [0,1]^k$ given by
\[I_x(r_1,\dots,r_k) = \pa{\bra{F_{x, r_1,\dots, r_{i-1}}(r_i)}}_{i=1}^k\]
Notice that $I_x$ is invertible, as we may write inductively, for any $(c_i) \in [0,1]^k$,
\[I_x^{-1}(c_1,\dots,c_k)_i = \begin{cases}F^{-1}_{x}(c_1) & i = 1\\F^{-1}_{x, I_x^{-1}(c_1,\dots,c_k)_1, \dots, I_x^{-1}(c_1,\dots,c_{k})_{i-1}}(c_i) & \text{else}\end{cases}\]
We then define the random variable
\begin{align*} 
R'(\omega) &:= I_{X(\omega)}(R(\omega))\\
&= \pa{\bra{F_{X(\omega), R_1(\omega),\dots, R_{i-1}(\omega)} \circ R_i}(\omega)}_{i=1}^k \in [0,1]^k
\end{align*}
given by mapping each $R_i$ to its corresponding conditional quantile. As $F_{X(\omega), R_1(\omega),\dots, R_{i-1}(\omega)}$ is surjective over $[0,1]$, since it is continuous, we have that for any $\omega \in \Omega$ and $r \in \R$, the conditional cumulative distribution function 
\begin{align*}
    &F'_{X(\omega), R_1(\omega),\dots, R_{i-1}(\omega)}(r)\\
    &:=\bra{\mu_{X(\omega), R_1(\omega),\dots, R_{i-1}(\omega)} \circ R_i^{\prime -1}}\pa{\sett{z \in \R \mid z \le r}}\\
    &= \bra{\mu_{X(\omega), R_1(\omega),\dots, R_{i-1}(\omega)} \circ R_i^{-1} \circ F^{-1}_{X(\omega), R_1(\omega),\dots, R_{i-1}(\omega)}}\pa{\sett{z \in \R \mid z \le r}}\\
    &= \bra{\mu_{X(\omega), R_1(\omega),\dots, R_{i-1}(\omega)} \circ R_i^{-1}}\pa{\sett{z \in \R \mid z \le F^{-1}_{X(\omega), R_1(\omega),\dots, R_{i-1}(\omega)}(r)}}\\
    &= \bra{F_{X(\omega), R_1(\omega),\dots, R_{i-1}(\omega)} \circ F^{-1}_{X(\omega), R_1(\omega),\dots, R_{i-1}(\omega)}}(r)\\
    &= r
\end{align*}
whence it follows that the conditional value of $R'_i$ is uniformly distributed on the interval. We now show that $R'$ is independent from $X$. As the Borel $\sigma$-algebra $\Bor([0,1]^k) = \bigotimes_{i=1}^k\Bor([0,1])$, and is thus generated by products of intervals $\prod_{i=1}^k [0,c_i]$, it suffices to show that $R'$ is independent of $X$ for such events. This can be seen as for any $A \subset \Bor(K)$ and $\prod_{i=1}^k [0,c_i] \in \bigotimes_{i=1}^k\Bor(\R)$, we have that

\begin{align*}
    &\mu\pa{X^{-1}(S) \cap R^{\prime -1}\pa{\prod_{i=1}^k [0,c_i]}} \\
    &= \int_A \bra{\mu_x \circ R^{\prime -1}}\pa{\prod_{i=1}^k [0,c_i]}d(\mu \circ X^{-1})(x)\\
    &= \int_A \int_{[0,c_1]} \cdots \int_{[0,c_k]} d\bra{\mu_{x,r_1,\dots,r_{k-1}} \circ R_k^{\prime -1}}(r_k) \cdots d\bra{\mu_{x} \circ R_2^{\prime -1}}(r_1)d(\mu \circ X^{-1})(x)\\
    &= \int_A \pa{\prod_{i=1}^k c_i} d(\mu \circ X^{-1})(x)\\
    &= \pa{\prod_{i=1}^k c_i} \mu\pa{X^{-1}(A)}\\
    &= \mu\pa{X^{-1}(A)} \int_K \int_{[0,c_1]} \cdots \int_{[0,c_k]} d\bra{\mu_{x,r_1,\dots,r_{k-1}} \circ R_k^{\prime -1}}(r_k) \cdots d\bra{\mu_{x} \circ R_2^{\prime -1}}(r_1)d(\mu \circ X^{-1})(x)\\
    &= \mu\pa{X^{-1}(A)} \int_K \bra{\mu_x \circ R^{\prime -1}}\pa{\prod_{i=1}^k [0,c_i]}d(\mu \circ X^{-1})(x)\\
    &= \mu\pa{X^{-1}(A)}\mu\pa{R^{\prime -1}\pa{\prod_{i=1}^k [0,c_i]}}
\end{align*}
as desired. We thus get a decomposable factorization $L = T' \circ W'$ where $W': \Omega \to K \times [0,1]^k$ is given by
\[W'(\omega) := (X(\omega), R'(\omega))\]
and $T': K \times [0,1]^k \to \Theta$ is given by
\[T'(x, c_1, \dots, c_k) := T(x, I_x^{-1}(c_1,\dots,c_k))\]
\end{proof}

\newpage

\section{Proof of Lemma \ref{lem:decompose}}\label{proof:decompose}
\textbf{Lemma \ref{lem:decompose} }\textit{Let $L = T \circ W$ be decomposable.  For any $x \in K$,
\[\mu_x \circ L^{-1} = \mu \circ R^{-1} \circ  T_x^{-1}\]
almost everywhere, where $\mu_x$ is a regular conditional probability for $\mu$ over $X$. }\\

\begin{proof}
Let $x \in K$ and let $\mu_x$ denote the corresponding conditional probability measure. Using the given factorization, we may write, for any $S \in \F_\Theta$,
\begin{align*}
    [\mu_x \circ L^{-1}](S) &= [\mu_x \circ W^{-1} \circ T^{-1}](S)\\
    &= [\mu_x \circ W^{-1}]\pa{T^{-1}(S)}\\
    &= [\mu_x \circ X^{-1} \times \mu_x \circ R^{-1}](T^{-1}(S))\\
    &= \int_{b \in [\pi_\Gamma \circ T^{-1}](S)} \int_{\sett{a \in K \mid (a,b) \in T^{-1}(S)}}d[\mu_x \circ X^{-1}]d[\mu_x \circ R^{-1} \circ \pi_\Gamma] \\
    &= \int_{(a,b) \in T^{-1}(S)} (\mathbbm{1}_{x = a}) d[\mu_x \circ R^{-1} \circ \pi_\Gamma] \\
    &= [\mu_x \circ R^{-1} \circ \pi_\Gamma]\pa{\sett{b \in \Gamma \mid (x,b) \in T^{-1}(S)}}\\
    &= [\mu_x \circ R^{-1} \circ T_x^{-1}](S)
\end{align*}
To complete the proof, we show that for any subset $H \in \F_\Gamma$, \[[\mu_x \circ R^{-1}](H) = [\mu \circ R^{-1}](H)\]
Notice that for any such $H$, and any $A \in \Bor(K)$, we have that
\begin{align*}
    \int_A [\mu \circ R^{-1}](H)d[\mu \circ X^{-1}] &= [\mu \circ R^{-1}](H)[\mu \circ X^{-1}](A)\\
    &= [\mu \circ X^{-1} \times \mu \circ R^{-1}](A \times H)\\
    &= [\mu \circ W^{-1}](H \times A)\\
    &= \mu\pa{X^{-1}(A) \cap R^{-1}(H)}\\
    &= \int_A [\mu_x \circ R^{-1}](H)d[\mu \circ X^{-1}](x)
\end{align*}
by Definition \ref{def:condprob}. As $\Omega$ and $K$ are both Radon, regular conditional probability is almost everywhere unique, whence it follows by the above that 
\[\mu_x \circ L^{-1} = \mu_x \circ R^{-1} \circ T_x^{-1} = \mu \circ R^{-1} \circ T_x^{-1}\]
almost everywhere, thus proving the desired result. 
\end{proof}

\newpage

\section{Proof of Lemma \ref{lem:jordan}}\label{proof:jordan}
\textbf{Lemma \ref{lem:jordan} }\textit{Let $J \in \F_\Theta$ be a Jordan-measurable subset, and let $\sett{\nu_t}_{t \in T}$ be a family of probability measures, each satisfying $\nu_t \ll \lambda^d$ and $\frac{d\nu_t}{d\lambda^d} \le D$. Then,
\[\lim_{\delta \to 0} \sup \sett{\nu_t\sett{\theta \in \R^d \setminus J \mid d(\theta, J) < \delta} \mid t \in T} = 0\]}
\begin{proof}
Noting that the sets are nested, we can see that
\begin{align*}
    \lim_{\delta \to 0} \lambda^d\pa{\sett{\theta \in \R^d \setminus J \mid d(\theta, J) < \delta}} &= \lambda^d\pa{\bigcap_{\delta > 0} \sett{\theta \in \R^d \setminus J \mid d(\theta, J) < \delta}} = \lambda^d(\partial J) = 0
\end{align*}
as the set $J$ is Jordan-measurable. It then follows that for any $\epsilon > 0$, there exists a $\tau$ such that for all $\delta < \tau$, $\lambda^d\pa{\sett{\theta \in \R^d \setminus J \mid d(\theta, J) < \delta}} < \frac{\epsilon}{D}$. Thus, for any $t \in T$,
\begin{align*} 
 \nu_t\pa{\sett{\theta \in \R^d \setminus J \mid d(\theta, J) < \delta}} &= \int_{\sett{\theta \in \R^d \setminus J \mid d(\theta, J) < \delta}} \frac{d\nu_t}{d\lambda^d} d\lambda^d\\
& \le D \lambda^d\pa{\sett{\theta \in \R^d \setminus J \mid d(\theta, J) < \delta}}\\
& < \epsilon
\end{align*}
whence the desired result follows.
\end{proof}

\newpage
\section{Proof of Lemma \ref{lem:extend}}\label{proof:extend}
\textbf{Lemma \ref{lem:extend} }\textit{ For any Borel set $S \in \Bor(\R^d)$,
\[\inf \sett{\lambda^d\pa{S \Delta J} \mid J \in \Bor(\R^d), \lambda^d(\partial J) = 0} = 0\]}

\begin{proof}
We show that for any $\epsilon > 0$, there exists a Jordan-measurable set $J$ such that $\lambda^d(S \Delta J) < \epsilon$. To see this, note that by the definition of the Lebesgue measure, there exists a countable collection of $d$-dimensional boxes $\sett{U_i}_{i=1}^\infty$ that cover $S$ such that $\lambda^d\pa{S \Delta \bigcup_{i=1}^\infty U_i} < \frac{\epsilon}{2}$. Notice that each box $U_i$ is Jordan-measurable, and so is any finite union of boxes. Thus, letting $J := \bigcup_{i=1}^N U_i$, it suffices to show that for some $N$, $\lambda^d\pa{S \Delta \bigcup_{i=1}^N U_i} < \epsilon$. We choose $N$ such that the tail sum $\sum_{i = N+1}^\infty \lambda^d(U_i) < \frac{\epsilon}{2}$. 
\begin{align*} 
\lambda^d\pa{S \Delta \bigcup_{i=1}^N U_i} &\le \lambda^d\pa{S \Delta \bigcup_{i=1}^\infty U_i} + \lambda^d\pa{\bigcup_{i=1}^\infty U_i \Delta \bigcup_{i=1}^N U_i} \\
&< \frac{\epsilon}{2} + \sum_{i = N+1}^\infty \lambda^d(U_i)\\
&< \epsilon
\end{align*}
as desired.
\end{proof}
\end{document}